\documentclass{article}

\usepackage[accepted]{icml2025}

\usepackage{amsmath, amssymb, amsfonts, bm}
\usepackage{siunitx}
\DeclareSIUnit{\bp}{bp}

\usepackage{booktabs}
\usepackage{array}
\usepackage{makecell}
\usepackage{graphicx}
\usepackage{tikz}
\usetikzlibrary{arrows.meta}
\usepackage{natbib}
\usepackage[hidelinks]{hyperref}
\usepackage{enumitem}
\usepackage{makecell}
\usepackage{xurl}

\usepackage{amsthm}
\newtheorem{proposition}{Proposition}
\newtheorem{remark}{Remark}
\newtheorem{application}{Application}

\usepackage{float}

\icmltitlerunning{Re-evaluating Short- and Long-Term Trend Factors}

\makeatletter\def\Hy@Warning#1{}\makeatother

\begin{document}

\twocolumn[
\icmltitle{Re-evaluating Short- and Long-Term Trend Factors in CTA Replication: \\A Bayesian Graphical Approach}

\icmlsetsymbol{equal}{*}

\begin{icmlauthorlist}
\icmlauthor{Eric Benhamou}{aifa,dauph}
\icmlauthor{Jean-Jacques Ohana}{aifa}
\icmlauthor{Alban Etienne}{aifa}
\icmlauthor{Béatrice Guez}{aifa}
\icmlauthor{Ethan Setrouk}{aifa}
\icmlauthor{Thomas Jacquot}{aifa}
\end{icmlauthorlist}

\icmlaffiliation{aifa}{Ai For Alpha, Paris, France}
\icmlaffiliation{dauph}{Université Paris Dauphine–PSL}

\icmlcorrespondingauthor{Eric Benhamou}{eric.benhamou@aiforalpha.com}
\icmlkeywords{CTA Replication, Short, long-term trends, Replication, Graphical Models}
\vskip 0.3in
]

\printAffiliationsAndNotice{\mbox{}}

\begin{abstract}
Commodity Trading Advisors (CTAs) have historically relied on \emph{trend‑following} rules that operate on vastly different horizons—from long‑term breakouts that capture major directional moves to short‑term momentum signals that thrive in fast‑moving markets. Despite a large body of work on trend following, the relative merits and interactions of \emph{short-} versus \emph{long‑term} trend systems remain controversial. This paper adds to the debate by (i) dynamically decomposing CTA returns into \emph{short‑term trend}, \emph{long‑term trend} and \emph{market beta} factors using a Bayesian graphical model, and (ii) showing how the blend of horizons shapes the strategy’s risk‑adjusted performance.
\end{abstract}

\section{Introduction \& Motivation}
Professional trend followers range from traditional long-term 500-day breakout systems to intraday futures scalpers. Academic work likewise presents mixed conclusions about the effectiveness of trend-following strategies and the optimal trend horizons. For instance, \cite{Moskowitz2012TimeSeriesMomentum, Baltas2013momentum} document sizable profits from 1--12‑month trends, whereas \cite{Jegadeesh2022short} highlight the fragility of ultra-short signals, and \citet{Clenow2023following} advocates for very short term trends. Other industry studies emphasize the benefit of \emph{multi‑horizon} mixes \cite{Baz2015DissectingInvestmentStrategies,BaltasKosowski2020}. Practitioner commentary reflects both the current challenges these strategies face in regaining past success \citep{Kilburn2024Trend,tzotchev2018designing} and the promise of shorter-term approaches for delivering more dynamic and convex return profiles \citep{Kilburn2024Horizons}. This perspective is reinforced by the 2025 “tariff test” episode, during which a new variant—dubbed the risk-off CTA—demonstrated its ability to mitigate drawdowns amid sharp market reversals \citep{Bartholomew2025Novel}. Thought leaders like Cliff Asness also stress the inherently statistical nature of trend-following and its ongoing integration with advanced AI techniques \citep{Asness2024AI}. Despite these contributions, a unified quantitative framework for blending horizons remains undeveloped. To address this gap, we investigate four questions:

\begin{enumerate}[itemsep=2pt, parsep=0pt, topsep=0pt, partopsep=0pt, after=\vspace{-0.5cm}]
  \item Can CTA-style strategies be replicated using only liquid futures?
  \item How do short- and long-term trend factor risk–returns compare over the past decade?
  \item Does mixing horizons materially improve drawdowns and diversification?
  \item Are short-term trends valuable or should we rely on longer trend horizons?
\end{enumerate}

\paragraph{Key Contributions and Innovation.}
This study makes \emph{two} novel contributions to the CTA-replication literature:  
\begin{enumerate}[label=(\roman*), itemsep=2pt, parsep=0pt, topsep=0pt, partopsep=0pt, after=\vspace{-0.6cm}]
    \item We introduce a \textbf{Bayesian graphical model} that \emph{simultaneously} tracks short-, long-term and market-beta exposures at the \emph{contract level}.  
    \item We document how a \textbf{short-term-trend \(+\) raw-beta sleeve} outperforms classic multi-month breakouts on both Sharpe and drawdown efficiency over 2010–2025—even after realistic transaction and roll costs.  
\end{enumerate}

\paragraph{Structure of the Paper.}
The remainder of the paper is organized as follows.  
Section~\ref{sec:litreview} surveys related work on trend following, short-horizon momentum and CTA replication.  
Section~\ref{sec:factor-design} details the horizon-specific \emph{lookback-straddle} factor construction.  
Section~\ref{sec:theory} develops the analytical foundations, proving that the look-back-straddle delta operates as a drift filter and that the Sharpe-optimal blend of short- and long-horizon factors tilts toward the long-term component.
Section~\ref{sec:methodology} describes the Bayesian graphical model and estimation procedure.  
Section~\ref{sec:data_empirical_results} presents empirical results on factor performance, cost sensitivity and robustness.  
Section~\ref{sec:utility} evaluates utility-based trade-offs between replication accuracy and risk efficiency.  
Finally, Section~\ref{sec:conclusion} concludes and outlines ideas for future research; proofs, supplementary tables and full cost-calculation details appear in the Appendix.

\section{Literature Review}\label{sec:litreview}

\noindent\textbf{Classical trend following.}  
Post-crisis evidence continues to affirm the economic and statistical significance of \mbox{3--12-month} time-series momentum.  
Using large cross-asset panels, \citet{Lemperiere2017}, \citet{Baz2015DissectingInvestmentStrategies} and \citet{BaltasKosowski2020} document Sharpe ratios that are broadly comparable to the pre-2008 period.  

\vspace{6pt}
\noindent\textbf{Short-term momentum and reversals.}  
The economics of very short-horizon signals are far less stable.  
\citet{NarayanBannettFaff2015} show that daily momentum mean-reverts strongly once micro-structure frictions are accounted for, while \citet{BoonsPrado2019} and \citet{Jegadeesh2022short} attribute much of the apparent weekly alpha to bid–ask bounce and dealer-inventory effects.  
Adaptive filters that down-weight periods of elevated noise partly mitigate—but do not eliminate—this performance drag.

\vspace{6pt}
\noindent\textbf{CTA replication and factor models.}  
Responding to fee pressure, a growing strand of work evaluates cost-effective CTA replication.  
Linear factor mimickers (like in \citet{Braun2024pursuit}), or Kalman filter replication (like in  \citet{Benhamou2018trend,Benhamou2018Demystified}) achieve correlations of $0.70$–$0.80$ to the \emph{SG Trend} index while retaining most downside protection, implying that investors can capture the bulk of the systematic-trend payoff at a fraction of traditional management and incentive fees.

\vspace{6pt}
\noindent\textbf{Bayesian and machine-learning trend decomposition.}  
Recent studies frame trend signals within explicitly time-varying statistical structures.  
Bayesian state-space decompositions and deep neural networks (\citet{Gu2021autoencoder}; \citet{Zhang2019high}) uncover latent regime shifts in both signal strength and asset-class exposure.  
These approaches highlight that the trend premium is not static but oscillates with macro-volatility and liquidity conditions, opening the door to dynamic allocation across “trend regimes” rather than binary on/off timing rules.

\noindent While the post‑2015 literature recognises the value of mixing horizons, few papers explicitly decompose CTA returns into distinct short‑ and long‑term trend factors—leaving the debate unresolved.

\section{Trend Factor Design: Lookback Straddles}\label{sec:factor-design}
Inspired by the payout symmetry of a \emph{lookback straddle}—a combination of a lookback call and put that captures the \emph{ex‑post} maximum excursion of the underlying—we engineer horizon‑specific trend factors by replicating the \emph{delta} of a lookback straddle across multiple windows as detailed in \citet{Fung2001risk}.

For each contract price series $S_t$ and window length $n$ (in days) we define the running maximum, running minimum and volatility:
\begin{align}
  M_{t,n} &= \max_{t-n+1 \le j \le t} S_j, \label{eq:running_max}\\
  m_{t,n} &= \min_{t-n+1 \le j \le t} S_j, \label{eq:running_min} \\
  \sigma_n^2     &= \frac{1}{n-1}
                    \sum_{j=t-n+1}^{t}
                      \Bigl[\ln\!\bigl(S_j / S_{j-1}\bigr)\Bigr]^2,
\end{align}

Let $\Phi(\cdot)$ denote the standard normal CDF.  The \textbf{lookback‑straddle trend score} for horizon $n$ is then
\begin{equation}
  T_{t,n} \;=\; \Phi\!\Bigl( \tfrac{\ln(S_t / m_{t,n})}{\sigma_n \sqrt{n}} \Bigr)
            - \Phi\!\Bigl( \tfrac{\ln(M_{t,n} / S_t)}{\sigma_n \sqrt{n}} \Bigr).
  \label{eq:trendScore}
\end{equation}
This score is \emph{convex} in extreme up or down moves, generating positively skewed payoffs reminiscent of managed‑futures “crisis alpha.” 

\paragraph{Short- vs. Long‑Horizon Sets.}  We implement two buckets of windows:
\begin{itemize}
  \item \textbf{Short‑Term (ST):} $n \in \{10,\,20,\,40,\,60\}$.  These capture fast breakouts and reversals.
  \item \textbf{Long‑Term (LT):} $n = 500$.  This captures the classic two-year long term trend followers.
\end{itemize}
For each $n$ we compute $T_{t,n}$ via~\eqref{eq:trendScore}.  
The composite trend factor is simply the mean across horizons:
\begin{equation}
  F^{\text{Trend}}_t \;=\;
  \frac{1}{|\mathcal N|}
  \sum_{n\in\mathcal N} T_{t,n},
\end{equation}
so that every look-back window contributes equally.

\paragraph{Complementary Market‑Beta Factor.}  Beyond the trend scores, we retain the \emph{raw daily return} for each contract, $r_t^{m}$, which will act as the market‑beta factor in Eq.~\eqref{eq:r_cta}.  Together, $\{T_{t,n}\}_{n\in\mathcal N}$ and $r_t^{m}$ span both directional momentum and baseline market exposure, enabling the Bayesian model to disentangle pure trend alpha from beta‑like moves.

\section{Theoretical Foundations}\label{sec:theory}
This section lays out the key analytical results in two stages:  
(i) we prove that the look-back–straddle trend score \(T_{t,n}\) equals the delta of a fixed-window look-back straddle under geometric Brownian motion, providing an option-pricing–based drift filter;  
(ii) we derive an explicit formula for the absolute Sharpe ratio of any portfolio that mixes short- and long-horizon trend factors, isolate the diversification term, and obtain a closed-form optimal weight that clarifies when short-term momentum should be added and confirms that the allocation should lean more heavily toward the long-term trend factor.



\subsection{Look-Back–Straddle Delta as a Drift Filter}
\label{subsec:LB_delta}

Assume that the log-price follow a geometric Brownian motion
\begin{equation}
\mathrm{d}X_t \;=\; \mu\,\mathrm{d}t + \sigma\,\mathrm{d}W_t, \qquad 
X_0 = x,
\end{equation}
with constant drift~$\mu$ and volatility~$\sigma$.  
For any fixed window of length~$n$ define the running maximum and minimum as in equations \eqref{eq:running_max} and \eqref{eq:running_min} and the \emph{fixed-window look-back straddle} maturing at~$t$ whose payoff is:  
\begin{equation}
L_{t,n} \;=\; M_t - m_t .
\end{equation}
Throughout this subsection we work under the risk-neutral measure
($\mu\!=\!0$) when pricing, and return to the physical measure when we
interpret the result as a drift filter.

\paragraph{Delta of the straddle.}
Under the risk-neutral measure, the fixed-window look-back straddle is valued as the discounted conditional expectation \(V(s,X_s,M_s,m_s)=e^{-r(t-s)}\mathbb{E}[\,M_t-m_t\,|\,X_s,M_s,m_s]\); a closed-form expression is provided by \citet{Goldman1979path,Conze1991path}).  Differentiating that price with respect to the current log-price and letting \(s\rightarrow t\) yields the delta
\begin{equation}\label{eq:LB_delta}
\Delta_{t,n}
  = \Phi\!\Bigl(\tfrac{X_t-m_t}{\sigma\sqrt{n}}\Bigr)
    - \Phi\!\Bigl(\tfrac{X_t-M_t}{\sigma\sqrt{n}}\Bigr),
\end{equation}
which coincides exactly with our trend score \(T_{t,n}\), thereby interpreting the score as an option-pricing–based drift filter.

\paragraph{Equality with the trend score.}
Equation~\eqref{eq:trendScore} of the main text defines the
look-back–straddle \emph{trend score}
\(
T_{t,n}
  \;=\;
  \Phi\!\bigl(z_t^+\bigr) - \Phi\!\bigl(z_t^-\bigr),\;
  z_t^\pm := \tfrac{X_t-m_t}{\sigma\sqrt{n}},\;
               \tfrac{X_t-M_t}{\sigma\sqrt{n}},
\)
which is identical to the right-hand side of~\eqref{eq:LB_delta}.
Hence
\begin{equation}
T_{t,n} = \Delta_{t,n},
\end{equation}
establishing that the proposed trend score is precisely the \emph{delta}
of a fixed-window look-back straddle.

\paragraph{Interpretation as a drift filter.}
Because the delta of any option captures the instantaneous change in its
fair value per unit change in the underlying, $T_{t,n}$ inherits two key
properties:

\begin{enumerate}[label=(\alph*)]
\item \textbf{Scale-free sensitivity.}  Through the
      $z$-scores $z_t^\pm$ the filter rescales distances to the
      running extrema by the diffusion scale $\sigma\sqrt{n}$, making
      it comparable across assets and horizons.
\item \textbf{Directional signal.}  When $X_t$ is near the upper
      (lower) boundary, $T_{t,n}\!\approx\!+1$ ($-1$), signalling a
      strong upward (downward) drift over the past~$n$ days, whereas a
      value near $0$ indicates sideways price action.
\end{enumerate}

Thus the look-back–straddle trend score functions as an
option-pricing–based, horizon-specific drift detector whose magnitude is
directly linked to the expected P\&L of a dynamically delta-hedged
straddle over the window~$n$.

\subsection{Sharpe–Ratio Decomposition}

Let $F^{\mathrm{ST}}$ and $F^{\mathrm{LT}}$ be the short– and long–horizon factors and assume that there are correlated:
\begin{equation}
\rho \;=\;\operatorname{corr}\bigl(F^{\mathrm{ST}},F^{\mathrm{LT}}\bigr).
\end{equation}
We form a one–unit–budget portfolio
\begin{equation}
P \;=\;\omega_{\mathrm{ST}}\,F^{\mathrm{ST}}
      +\omega_{\mathrm{LT}}\,F^{\mathrm{LT}},
\qquad
\omega_{\mathrm{ST}}+\omega_{\mathrm{LT}}=1.
\end{equation}
Denote $\mu_i=\mathbb E[F^i]$ and 
\(\Sigma=\begin{pmatrix}1&\rho\\\rho&1\end{pmatrix}\), the first two moments of the two short– and long–horizon factors. Then we can compute \(S(\omega)\), the \emph{absolute} Sharpe ratio (i.e.\ the Sharpe ratio assuming a zero risk-free rate, equivalently the ratio of expected return to volatility), as follows:

\begin{eqnarray}
\mathcal S(\omega)
&=&\frac{\mathbb E[P]}{\sqrt{\operatorname{Var}(P)}}
= \frac{\omega_{\mathrm{ST}}\mu_{\mathrm{ST}}
       +\omega_{\mathrm{LT}}\mu_{\mathrm{LT}}}
      {\sqrt{\omega^{\top}\Sigma\,\omega}} \\
&=& \frac{\omega_{\mathrm{ST}}\mu_{\mathrm{ST}}
       +\omega_{\mathrm{LT}}\mu_{\mathrm{LT}}}
      {\sqrt{\omega_{\mathrm{ST}}^2
            +\omega_{\mathrm{LT}}^2
            +2\,\rho\,\omega_{\mathrm{ST}}\omega_{\mathrm{LT}}}}.
\end{eqnarray}

Using $\omega_{\mathrm{LT}}=1-\omega_{\mathrm{ST}}$ one checks that
\begin{equation}
\omega_{\mathrm{ST}}^2+\omega_{\mathrm{LT}}^2
 =1-2\,\omega_{\mathrm{ST}}\omega_{\mathrm{LT}},
\end{equation}
so that
\begin{equation}
\operatorname{Var}(P)
=1-2\,\omega_{\mathrm{ST}}\omega_{\mathrm{LT}}+2\,\rho\,\omega_{\mathrm{ST}}\omega_{\mathrm{LT}}
=1+2(\rho-1)\,\omega_{\mathrm{ST}}\omega_{\mathrm{LT}},
\end{equation}
i.e.\ the familiar $2(\rho-1)\omega_{\mathrm{ST}}\omega_{\mathrm{LT}}$ variance‐reduction from diversification.

\begin{proposition}\label{prop:optimum_weight}
The weight \(\omega_{\mathrm{ST}}\) that maximizes the \textit{absolute} Sharpe ratio (as we do not subtract in the numerator the risk free rate) of the portfolio denoted by \(\mathcal S(P)\), subject to the budget constraint \(\omega_{\mathrm{ST}} + \omega_{\mathrm{LT}} = 1\), is given by
\begin{equation}\label{eq:optimum}
\omega_{\mathrm{ST}}^*
=\frac{\mu_{\mathrm{ST}}-\rho\,\mu_{\mathrm{LT}}}
      {(\mu_{\mathrm{ST}}+\mu_{\mathrm{LT}})(1-\rho )},
\end{equation}
\end{proposition}

\begin{proof}
Differentiating \(\mathcal S(P) \) with respect to \(\omega_{\mathrm{ST}}\) and enforcing \(\omega_{\mathrm{LT}}=1-\omega_{\mathrm{ST}}\) yields the closed‐form optimum.  The detailed derivation is given in Appendix~\ref{app:proof_optimum}.
\end{proof}

\begin{remark}[Short- vs.\ Long-Term Trends]
The Sharpe-maximizing weight is strictly positive whenever \(
\mu_{\mathrm{ST}} > \rho\,\mu_{\mathrm{LT}},
\). Moreover, if \(\mu_{\mathrm{ST}}\) and \(\mu_{\mathrm{LT}}\) are equal, then \(\omega_{\mathrm{ST}}^* = 50 \%\), implying equal weighting of the short- and long-term trend factors.
\end{remark}

\begin{application}
Using Proposition~\ref{prop:optimum_weight} with the sample moments from Table~\ref{tab:perf_summary} gives
\[
\omega_{\mathrm{ST}}^* \approx 17\%, \hspace{2cm}
\omega_{\mathrm{LT}}^* \approx 83\%,
\]
confirming that one should allocate a larger weight (almost five times i.e. $ 83 \div 17$) to the Long-Term than to the Short-Term trend factor.
\end{application}

\section{Methodology}\label{sec:methodology}
We model the daily excess return of the SG CTA benchmark, $r_{t}^{\text{CTA}}$, as a \emph{market-by-market} mixture of three horizon-specific factors:
\begin{align} \label{eq:r_cta}
  r_{t}^{\text{CTA}} 
    \!= \! \sum_{m\in\mathcal M}  \Bigl( & \beta_{t,m}^{\text{ST}} f_{t,m}^{\text{ST}}
      \!+\! \beta_{t,m}^{\text{LT}} f_{t,m}^{\text{LT}}
      \!+\! \beta_{t,m}^{\text{MKT}} r_{t}^{m} \Bigr)
      \!+\! \varepsilon_{t},\\
  (\beta_{t,m}^{\cdot}\mid\mathcal F_{t-1})
    &\sim \mathcal N\!\bigl(\beta_{t-1,m}^{\cdot},\, \sigma_{\beta}^{2}\bigr),
    \, \varepsilon_{t}\sim\mathcal N(0,\sigma_{\varepsilon}^{2}).
\end{align}
Here $m$ indexes the 24 liquid futures in the investment universe—spanning equity indices, government bond futures, major currency pairs, and key commodity contracts— and each $\beta_{t,m}^{\cdot}$ is a \emph{time‑varying, contract‑specific} loading.  The factors for a given market $m$ are constructed as follows:
\begin{itemize}
  \item $f_{t,m}^{\text{ST}}$ — short-term trend score of market $m$ (look-back straddle, horizons 10, 20, 40 and 60-day window),
  \item $f_{t,m}^{\text{LT}}$ — long-term trend score of market $m$ (500-day window),
  \item $r_{t}^{m}$ — raw daily return of market $m$ (captures market beta).
\end{itemize}
This formulation separates \emph{cross-asset allocation} (over $m$) from \emph{horizon allocation} (ST vs.\ LT vs.\ Market), letting the Bayesian filter infer both dimensions simultaneously.

\subsection{Graphical Model Primer}
Graphical models provide a transparent representation of conditional dependencies among variables, encoding them as nodes (random variables) and edges (probabilistic links).  Our dynamic factor model can be written as a time‑unrolled Bayesian network (Figure~\ref{fig:graphical_model}), with two node layers per date: latent \emph{states} (factor weights) and observable \emph{emissions} (CTA returns).

To illustrate the machinery, consider the more granular asset‑class version used for live daily decoding.  Let $\widehat{\text{NAV}}_t$ be the estimated net‑asset value on day $t$ and $r^{k}_t$ the return of asset class \( k \in \{\text{Eq}, \text{Fx}, \text{Bd}, \text{Co}\} \), where \(\text{Eq}\) denotes an Equity Index Future, \(\text{Bd}\) a Bond Future, \(\text{Fx}\) a Currency Future, and \(\text{Co}\) a Commodity Future. Denote their time‑$t$ weights by $w^{k}_t$.  The generative relationship is
\begin{equation}
  \widehat{\text{NAV}}_t 
  = \widehat{\text{NAV}}_{t-1}\Bigl(1 + \sum_{k} w^{k}_{t-1} r^{k}_t\Bigr),
  \label{eq:NAV-dynamics}
\end{equation}
mirroring a one‑step portfolio update in continuous compounding.  The weights follow a first‑order Gaussian random walk $w^{k}_{t}\sim\mathcal N(w^{k}_{t-1},\sigma_{w}^{2})$ and are \emph{a‑priori} correlated across asset classes, enabling interaction effects beyond a Kalman filter’s diagonal covariance.

\paragraph{Inference pipeline.}  Daily estimation proceeds via message‑passing: (i) a prediction step propagates the weight distribution forward, (ii) a correction step conditions on the new NAV observation, and (iii) backward smoothing refines the entire path.  This scheme is analogous to a Kalman filter but leverages full asset‑interaction structure.

\begin{enumerate}
  \item \textbf{State‑Space Representation:} latent weights $w^{k}_{t}$ constitute the hidden state; NAV returns are emissions.
  \item \textbf{Dynamic Inference:} FFBS or particle‑filtering yields the posterior $p(w_{1:T}\mid \text{NAV}_{1:T})$.
  \item \textbf{Interaction Modeling:} joint Gaussian priors across $k$ capture cross‑asset spill‑overs often ignored by independent factor models.
  \item \textbf{Online Updating:} as new prices arrive, the entire posterior is updated recursively, ensuring rapid adaptation.
\end{enumerate}

The graphical‑model framework thus unifies factor replication and style‑drift tracking in one coherent probabilistic engine—an approach long standard in speech recognition \citep{murphy2012machine} yet still under‑exploited in finance.

\begin{figure}[ht]
  \centering
  \resizebox{\columnwidth}{!}{%
    \begin{tikzpicture}[
      latent/.style   = {circle, draw, thick, fill=gray!30, minimum size=1cm},
      observed/.style = {circle, draw, thick, fill=white, minimum size=1cm},
      axis/.style     = {-Stealth, thick},
      link/.style     = {thick, shorten >=2pt, shorten <=2pt},
      cross/.style    = {-Stealth, dashed, draw=gray!55},
      every node/.style = {font=\small}
    ]

      \def\colW{3.5}
      \def\startY{-3}
      \def\dY{1.8}

      \draw[axis] (-1,0) -- (\colW*2+1.2,0) node[right]{Time};
      \foreach \t/\x in {1/0, 2/\colW, 3/\colW*2}
        \node at (\x,0.45) {$t=\t$};

      \foreach \i/\x in {0/0, 1/\colW, 2/\colW*2} {
        \node[observed] (NAV\i) at (\x,-1) {NAV};
        \ifnum\i>0
          \draw[axis] (NAV\the\numexpr\i-1\relax) -- (NAV\i);
        \fi
      }

      \foreach \i/\xShift in {0/0, 1/\colW, 2/\colW*2} {
        \node[latent] (Equity\i)    at (\xShift, \startY)                   {Eq};
        \node[latent] (Bond\i)      at (\xShift, \startY-\dY)               {Bd};
        \node[latent] (Currency\i)  at (\xShift, \startY-2*\dY)            {Fx};
        \node[latent] (Commodity\i) at (\xShift, \startY-3*\dY)            {Co};
      }

      \foreach \asset in {Equity,Bond,Currency,Commodity} {
        \foreach \i/\j in {0/1, 1/2} {
          \draw[link] (\asset\i) -- (\asset\j);
        }
      }

      \foreach \i in {0,1,2} {
        \draw[link] (NAV\i.south)    -- (Equity\i.north);
        \draw[link] (Equity\i.south) -- (Bond\i.north);
        \draw[link] (Bond\i.south)   -- (Currency\i.north);
        \draw[link] (Currency\i.south)-- (Commodity\i.north);
      }

      \foreach \i in {0,1,2} {
        \draw[cross] (Equity\i.east)   to[out=0,in=180] (Bond\i.west);
        \draw[cross] (Equity\i.east)   to[out=0,in=180] (Currency\i.west);
        \draw[cross] (Equity\i.east)   to[out=0,in=180] (Commodity\i.west);
        \draw[cross] (Bond\i.east)     to[out=0,in=180] (Currency\i.west);
        \draw[cross] (Bond\i.east)     to[out=0,in=180] (Commodity\i.west);
        \draw[cross] (Currency\i.east) to[out=0,in=180] (Commodity\i.west);
      }

    \end{tikzpicture}%
  }
  \caption{Simplified graphical model showing the relationship between observed NAV Returns and inferred allocation as time goes by. For illustration purposes, we use different assets, with one being an Equity Index Future shortened in Eq, a second one a Bond Future Factor shortened in Bd, a third one a Currency Future shortened in Fx, and finally a Commodity Future shortened in Co.}
  \label{fig:graphical_model}
\end{figure}
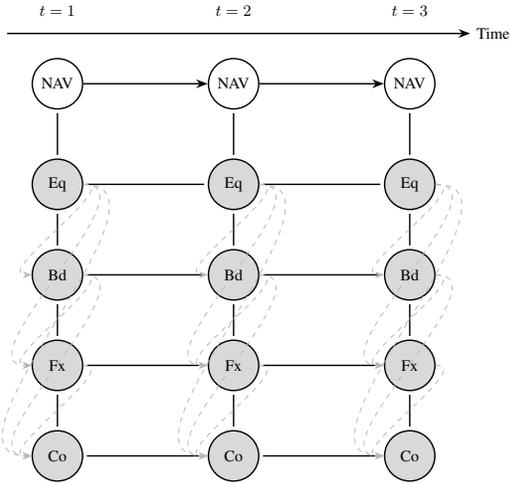

\subsection{Bayesian Filter as a Sparse Kalman Upgrade}

Setting $\Sigma_\beta=\sigma_\beta^2 I$ and dropping the cross-asset prior
correlations in Section~5 collapses our state evolution to a diagonal‐covariance
Gaussian random walk—\emph{exactly} the standard Kalman filter used in
\citet{Harvey1993Dynamic,Harvey199310StructuralTimeSeries}.  Retaining the full covariance matrix while placing
a sparsity-inducing Horseshoe prior on off-diagonal elements
(\citealp{CarvalhoPolsonScott2010}) lets the data decide which cross-horizon or
cross-asset interactions are worth tracking, yielding a filter that nests the
Kalman as a special case yet requires only $\mathcal{O}(K^2)$ rather than
$\mathcal{O}(K^3)$ operations thanks to block-diagonal structure.

\paragraph{Take-away.}  
The analytical results underpin the empirical findings: (i) the look-back
delta naturally filters drift; (ii) Sharpe improvements stem from both timing
\emph{and} convexity; (iii) the Bayesian network offers a computationally
tractable, theoretically sound generalization of classical linear filters.

\section{Data \& Empirical Results }\label{sec:data_empirical_results}

\subsection{Market Futures Universe}
Table~\ref{tab:futures_universe} lists the 24 exchange‑traded futures used in the study and the representative basis‑point costs we assume for each.

\begin{table}[!htbp]
  \centering
  \footnotesize
  \resizebox{\columnwidth}{!}{%
    \begin{tabular}{@{}lllrr@{}}
      \toprule
      Class & Contract (root) & Exchange & \multicolumn{1}{c}{Tx.~Cost (bp)} & \multicolumn{1}{c}{Roll Drag (bp)} \\
      \midrule
      \multicolumn{5}{@{}l}{\textit{Equity Index Futures}} \\
      & S\&P 500 (ES)        & CME     & \SI{2}{\bp}  & \SI{15}{\bp} \\
      & Nasdaq (NQ)         & CME     & \SI{2}{\bp}  & \SI{15}{\bp} \\
      & Nikkei (NK)         & JPX–OSE & \SI{2}{\bp}  & \SI{15}{\bp} \\
      & EuroStoxx (SX)      & Eurex   & \SI{2}{\bp}  & \SI{15}{\bp} \\
      & FTSE (Z)            & ICE     & \SI{2}{\bp}  & \SI{15}{\bp} \\
      & MSCI EM (EM)        & CME     & \SI{2}{\bp}  & \SI{15}{\bp} \\
      \addlinespace
      \multicolumn{5}{@{}l}{\textit{Fixed-Income Futures}} \\
      & US 2Y (TU)          & CBOT    & \SI{2}{\bp}  & \SI{10}{\bp} \\
      & Schatz 2Y (SZ)      & Eurex   & \SI{2}{\bp}  & \SI{10}{\bp} \\
      & US 10Y (TY)         & CBOT    & \SI{2}{\bp}  & \SI{10}{\bp} \\
      & Bund (RX)           & Eurex   & \SI{2}{\bp}  & \SI{10}{\bp} \\
      & Gilt (G)            & ICE     & \SI{2}{\bp}  & \SI{10}{\bp} \\
      & JGB 10Y (JGB)       & JPX–OSE & \SI{2}{\bp}  & \SI{10}{\bp} \\
      & Aus 10Y (XM)        & ASX     & \SI{2}{\bp}  & \SI{10}{\bp} \\
      & Can 10Y (CGB)       & MX      & \SI{2}{\bp}  & \SI{10}{\bp} \\
      \addlinespace
      \multicolumn{5}{@{}l}{\textit{Currency Futures}} \\
      & EURUSD (6E)         & CME     & \SI{2}{\bp}  & \SI{2}{\bp}  \\
      & JPYUSD (6J)         & CME     & \SI{2}{\bp}  & \SI{2}{\bp}  \\
      & GBPUSD (6B)         & CME     & \SI{2}{\bp}  & \SI{2}{\bp}  \\
      & AUDUSD (6A)         & CME     & \SI{2}{\bp}  & \SI{2}{\bp}  \\
      & CADUSD (6C)         & CME     & \SI{2}{\bp}  & \SI{2}{\bp}  \\
      \addlinespace
      \multicolumn{5}{@{}l}{\textit{Commodity Futures}} \\
      & Gold (GC)           & COMEX   & \SI{2}{\bp}  & \SI{15}{\bp} \\
      & WTI (CL)            & NYMEX   & \SI{2}{\bp}  & \SI{15}{\bp} \\
      & Brent (CO)          & ICE     & \SI{2}{\bp}  & \SI{15}{\bp} \\
      & NatGas (NG)         & NYMEX   & \SI{2}{\bp}  & \SI{15}{\bp} \\
      & Copper (HG)         & COMEX   & \SI{2}{\bp}  & \SI{20}{\bp} \\
      \bottomrule
    \end{tabular}
  }
  \caption{Contract roots, venues, and representative basis‑point costs per round‑turn (\textbf{Tx.~Cost}) and per roll cycle (\textbf{Roll Drag}).}
  \label{tab:futures_universe}
\end{table}

\subsection{Cost Framework (summary)}\label{sec:cost-summary}
We model three layers of implementation cost:

\begin{enumerate}[label=(\alph*), itemsep=2pt, topsep=2pt]
   \item \textbf{Transaction cost (Tx.~Cost).} Round-turn execution expense that bundles bid–ask, brokerage, exchange and clearing fees plus a small slippage buffer; see Table~\ref{tab:futures_universe} for contract-level figures.
  \item \textbf{Replication (roll) cost.} Systematic drag when the front-month future is rolled; measured as the 2000–2025 average front-to-next calendar spread.
  \item \textbf{Management fee.} Flat \(0.5\%\) per-annum charge on AUM.
\end{enumerate}

A complete derivation of all-in costs—including leverage assumptions, per-year round-turn counts and the resulting 1.38–1.68 \% p.a.\ range is provided in Appendix \ref{app:cost-details}.

\subsection{The Various Strategies}\label{subsec:the_various_strategies}

In this part, we compare seven factor‐based portfolios plus the SG CTA Trend Index benchmark. Each strategy is built from our horizon‐specific lookback‐straddle scores with/or without the raw market returns factors  and is identified by the following letter codes:

\begin{itemize}
  \item \textbf{LTT}: The \emph{Long‐Term Trend} factor averages lookback‐straddle deltas over a 500-day window to capture persistent directional moves, emulating classic two‐year breakout systems.
  
  \item \textbf{MKT}: The \emph{Markets} factor aggregates the raw daily returns of all liquid futures contracts, serving as a pure market‐beta proxy without momentum filtering.
  
  \item \textbf{STT}: The \emph{Short‐Term Trend} factor averages lookback‐straddle deltas across 10, 20, 40 and 60‐day horizons, targeting high‐frequency breakouts and rapid market shifts.
  
  \item \textbf{STT+LTT}: This \emph{Multi‐Horizon Trend} (denoted by STT+LTT) factor combines short- and long-term trends on an equal‐weight basis to capture both fast and slow trend dynamics.
  
  \item \textbf{MKT+STT+LTT}: The \emph{Markets + Multi‐Horizon Trend} factor allocates equally to raw market returns, short- and long-term trends, integrating baseline beta with diversified momentum drivers.
  
  \item \textbf{MKT+STT}: The \emph{Markets + Short‐Term Trend} factor pairs raw returns with only the short‐horizon trend factor, emphasizing rapid trend capture alongside fundamental market exposure.

  \item \textbf{SGCTAT}: The \emph{SG CTA Trend index} with ticker NEIXCTAT, is the industry benchmark for systematic trend-following. Maintained by Société Générale Prime Brokerage, it tracks the net performance of the 10 largest, diversified CTAs trading primarily futures, selected based on correlation to peers and daily return reporting. The index is equally weighted, rebalanced annually, and has been live since January 2000.
 
\end{itemize}

\subsection{Correlation Analysis}
We begin by examining the correlations among the \emph{outputs} of our Bayesian filter—i.e., the return streams inferred when using different combinations of inputs—rather than the raw inputs themselves. Table~\ref{tab:corr_matrix} presents the pairwise correlations between each of the seven factor portfolios and the SGCTAT over the evaluation period.

\begin{itemize}
  \item \textbf{Market vs.\ Long-Term Trend.} When the filter uses only \emph{raw market returns}, the inferred series correlates at 0.82 with the series obtained when \emph{Long-Term Trend (LTT)} factors are also included. This high correlation implies that the market-beta exposure already captures most of the slow-moving trend information, and that the standalone LT factor contributes only a modest incremental alpha.
  
  \item \textbf{Short-Term Trend’s Diversification Benefit.} The \emph{Short-Term Trend (STT)} series exhibits much lower correlations—0.24 with the MKT run and 0.50 with the LTT run. This limited overlap underpins STT’s greatest value: its convex payoff profile both cushions portfolios during sudden price swings and allows for quicker adaptation to regime shifts, despite STT’s relatively lower standalone Sharpe ratio.
\end{itemize}

\begin{table}[ht]
  \centering
  \footnotesize
  \resizebox{\columnwidth}{!}{%
    \begin{tabular}{@{}l*{7}{S[table-format=1.3]}@{}}
    \toprule
    Strategy &
    \parbox[t]{0.7cm}{\centering LTT} &
    \parbox[t]{0.7cm}{\centering MKT} &
    \parbox[t]{0.7cm}{\centering STT\\+LTT} &
    \parbox[t]{0.7cm}{\centering STT} &
    \parbox[t]{0.8cm}{\centering MKT\\+STT\\+LTT} &
    \parbox[t]{0.7cm}{\centering MKT\\+STT} &
    \parbox[t]{0.7cm}{\centering SG\\CTAT} \\
    \midrule      
      LTT          & 1.00 & {}   & {}   & {}   & {}   & {}   & {}   \\
      MKT          & 0.82 & 1.00 & {}   & {}   & {}   & {}   & {}   \\
      STT+LTT      & 0.91 & 0.70 & 1.00 & {}   & {}   & {}   & {}   \\
      STT          & \textbf{0.50} & \textbf{0.24} & 0.76 & 1.00 & {}   & {}   & {}   \\
      MKT+STT+LTT  & 0.92 & 0.77 & 0.96 & 0.69 & 1.00 & {}   & {}   \\
      MKT+STT      & 0.85 & 0.83 & 0.92 & 0.67 & 0.94 & 1.00 & {}   \\
      SGCTAT       & 0.81 & 0.65 & 0.84 & 0.65 & 0.85 & 0.80 & 1.00 \\
      \bottomrule
    \end{tabular}
  }
  \caption{Lower triangle of the correlation matrix between factor portfolios and the SG CTA Trend Index benchmark returns over the evaluation period.}
  \label{tab:corr_matrix}
\end{table}

\subsection{Performance Summary}
We now turn to the empirical results, comparing each factor portfolio’s risk–return profile over the full evaluation period. Table~\ref{tab:perf_summary} presents cumulative and annualized returns, volatility, Sharpe ratios, maximum drawdowns, and efficiency metrics (Return/MaxDD and Sharpe/MaxDD). Bold values highlight the best-performing strategy for each metric.

\begin{table}[ht]
  \centering
  \footnotesize
  \resizebox{\columnwidth}{!}{%
  \begin{tabular}{@{}l *{7}{S[table-format = 2.2]} @{}}
    \toprule
    {Metric} &
    \parbox[t]{0.7cm}{\centering LTT} &
    \parbox[t]{0.7cm}{\centering MKT} &
    \parbox[t]{0.7cm}{\centering STT\\+LTT} &
    \parbox[t]{0.7cm}{\centering STT} &
    \parbox[t]{0.8cm}{\centering MKT\\+STT\\+LTT} &
    \parbox[t]{0.7cm}{\centering MKT\\+STT} &
    \parbox[t]{0.7cm}{\centering SG\\CTAT} \\
    \midrule      
    Cumulative Return (\%)     & 80.80 & 92.10 & 77.20 & 47.20 & 92.30 & 99.30 & 26.30 \\
    Annual Return (\%)         &  6.10 &  6.70 &  5.90 &  3.90 &  6.80 &  7.10 &  2.40 \\
    Volatility (\%)            & 10.20 & 11.50 &  9.60 &  9.20 & 10.40 & 10.20 & 11.00 \\
    Sharpe Ratio               &  0.39 &  0.40 &  0.40 &  0.20 &  0.45 &  \textbf{0.49} & 0.03 \\
    Max Drawdown (\%)          & 18.80 & 20.30 & 16.70 & 15.20 & 17.70 & \textbf{14.90} & 22.40 \\
    Return/MaxDD               &  0.32 &  0.33 &  0.35 &  0.26 &  0.38 &  \textbf{0.48} & 0.11 \\
    Sharpe/MaxDD               &  2.09 &  1.98 &  2.37 &  1.34 &  2.53 &  \textbf{3.29} & 0.14 \\
    \bottomrule
  \end{tabular}%
  }
  \caption{Risk and return statistics for trend and market factor portfolios (daily sampling). Bold numbers indicate the best metric in each row.}
  \label{tab:perf_summary}
\end{table}

\subsection{Key Take-Aways}

\begin{itemize}
  \item \textbf{Short-Term Trend (STT).} On its own, the STT factor underperforms (Sharpe\,=\,0.20; Return/MaxDD\,=\,0.26). Frequent whipsaws lead to “false starts,” so STT primarily adds diversification rather than standalone alpha.
  \item \textbf{Long-Term Trend (LTT).} The LTT factor improves risk-adjusted returns (Sharpe\,=\,0.39) relative to STT but still lags mixed-horizon approaches.
  \item \textbf{Multi-Horizon Blend (STT+LTT).} Combining STT and LTT raises the Sharpe/MaxDD efficiency to 2.37, reflecting the asymmetric payoff benefits of mixing fast and slow trend signals.
  \item \textbf{MKT + STT.} Integrating raw market returns with the short-term trend factor yields the strongest performance: highest Sharpe (0.49), best Return/MaxDD (0.48), and shallowest max drawdown (14.9\%).
\end{itemize}

\paragraph{Extended‑sample evidence.} To stress‑test robustness we replicate the analysis over two additional horizons: (i) the live five‑year window from June 2020 to June 2025 and (ii) the full back‑test from January 2010 to June 2025. The results (see Appendix \ref{app:perf_tables}) confirm that the \textbf{MKT+STT} sleeve consistently offers the best risk‑return trade‑off:
\begin{itemize}
\item \emph{Five‑year window (2020–2025).} MKT+STT delivers the highest Sharpe/MaxDD (3.05) while capping the maximum drawdown below 15\%.  Its Return/MaxDD efficiency (0.53) outstrips all other sleeves and almost doubles the benchmark’s 0.27.
\item \emph{20‑year back‑test (2004–2025).} The same sleeve maintains leadership with a Sharpe/MaxDD of 4.94 and a Return/MaxDD of 0.64, demonstrating that the short‑horizon/market‑beta blend is not a recent artefact but a persistent source of convex carry.
\end{itemize}
Performance rises—both outright and on a tail-adjusted basis—as the sampling
window lengthens, indicating that short-horizon signals and raw-return beta
capture complementary premia that compound over time. Appendix~B makes the
point clear: Table~\ref{tab:perf_2010_2015} (trend-friendly 2010–15) and
Table~\ref{tab:perf_2016_2020} (choppy 2016–20) show the performance lead
shifting from the mixed-horizon \textbf{STT+LTT} sleeve to \textbf{MKT+STT}
once regimes turn noisy. \emph{Short-term (ST) signals} therefore serve
two roles—precision timing when trends are smooth and rapid de-risking when
they fragment—making them the chief source of convexity and drawdown
protection in volatile periods.

\section{Evaluating strategies through performance and correlation}\label{sec:utility}

\subsection{Motivation}
When allocating to a replication strategy, investors generally weigh two key issues simultaneously:  \\
(1) \emph{How much profit does the strategy generate per unit of risk?}\
(2) \emph{How closely does it track the benchmark it aims to replicate?}\\

\noindent We measure the first aspect by the ratio
\(y=\mathrm{Return}/\mathrm{MaxDD}\).
It captures the capacity to create gains relative to the single worst loss an
investor would have endured.  
The second aspect is simply the linear correlation
\(\rho\) with the SG~CTA Trend index, which we \emph{also} wish to maximise,
because higher correlation means better replication quality.

\noindent Our goal is therefore to decide whether one strategy is \emph{better},
\emph{worse} or \emph{equivalent} to another, given
the \((\rho,y)\)-coordinates of each point on the chart.

\subsection{Utility functions as a comparison tool}

To aggregate the two desirable attributes in a single preference score we use
bivariate utility functions.  Two classical forms suffice:

\paragraph{\textbf{Cobb–Douglas} (first formulation, 1927).}
For positive inputs \(x_1,x_2\) the Cobb–Douglas utility is
\(U_{\text{CD}}(x_1,x_2)=A\,x_1^{\alpha}x_2^{1-\alpha}\) with \(0<\alpha<1\).
It is homothetic, exhibits unit elasticity of substitution and assigns a clear
economic meaning to \(\alpha\): the relative weight of the first attribute.

\paragraph{CES.}
The Constant-Elasticity-of-Substitution family generalises this idea to
\(U_{\text{CES}}(x_1,x_2)=\bigl(\alpha x_1^{\rho}+(1-\alpha)x_2^{\rho}\bigr)^{1/\rho}\),
where \(\rho\) tunes the elasticity
\(\sigma=1/(1-\rho)\); Cobb–Douglas is the limiting case \(\rho\to0\).\\

\noindent Because we only need a single, transparent parameter to pivot the trade-off,
the Cobb–Douglas specification
\(U(\rho,y)=\rho^{\alpha}y^{1-\alpha}\)
is retained for the rest of the analysis.

\subsection{Iso-utility curves}

Fixing a utility level \(U\) gives the indifference locus
\[
  U=\rho^{\alpha}y^{1-\alpha}
  \quad\Longrightarrow\quad
  y = U^{1/(1-\alpha)}\,
      \rho^{-\alpha/(1-\alpha)},
\]
a smooth, strictly convex line in the \((\rho,y)\)-plane.  Moving
north-east across such a curve always raises utility.

\subsection{Pareto-dominant points on the chart}

Among the strategies displayed, \textbf{MKT+STT}
\((\rho=0.80,\,y=0.47)\) and
\textbf{MKT+STT+LTT} \((\rho=0.85,\,y=0.38)\)
dominate all other candidates: no alternative offers a higher
correlation or a higher Return/Max-DD ratio (excluding the SGCTAT benchmark, which by definition has a correlation of 1 with itself).  These two points therefore lie
on the \emph{Pareto frontier} of the sample.

\subsection{The \texorpdfstring{$\alpha$}{alpha} that makes the two dominant strategies indifferent}
Setting \(U(\rho_1,y_1)=U(\rho_2,y_2)\) for the pair above yields a single
solution
\[
  \alpha
  =\frac{\ln(y_2/y_1)}{\ln(y_2/y_1)+\ln(\rho_1/\rho_2)}
  \simeq 0.78,
  \qquad
  U\simeq0.71 .
\]
Hence one—and only one—Cobb–Douglas indifference curve passes through
both Pareto-dominant strategies.

\begin{figure}[ht]
  \centering
  \hspace*{-1em} 
  \includegraphics[width=1.05 \columnwidth]{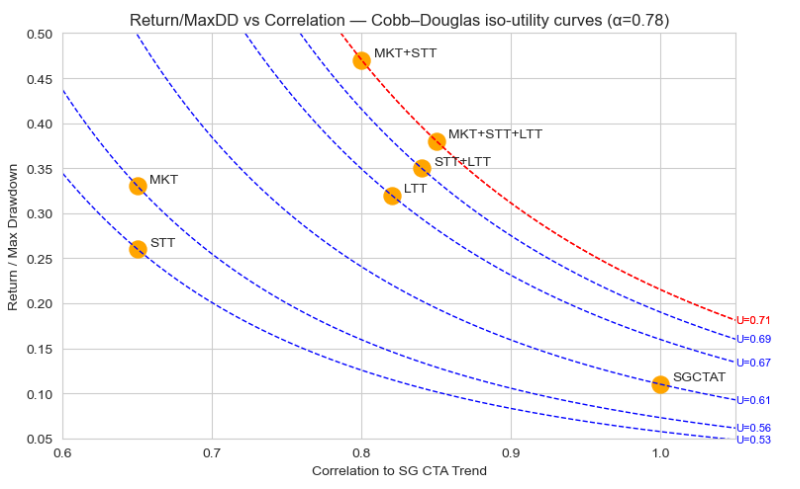}
  \caption{Cobb–Douglas indifference curve
           (\(\alpha = 0.78,\; U \approx 0.71\))
           that passes through the two Pareto-dominant strategies
           \textbf{MKT+STT} and \textbf{MKT+STT+LTT}
           in the \((\rho,\;\mathrm{Return}/\mathrm{MaxDD})\)-plane.}
\end{figure}

In other words, the two strategies are indifferent only if roughly 78\% of the utility weight is placed on correlation.

\subsection{Illustration with varying exponents}\label{subsec:illustration}
The figure below shows several iso-utility lines for different values of
\(\alpha\).  As \(\alpha\) increases, the curves pivot to favour correlation
over pure performance; as \(\alpha\) decreases, they tilt to reward
Return/Max-DD more heavily.  This visual device allows us to rank any new
strategy by simple inspection of its position relative to the chosen
iso-utility contours.

\begin{figure}[ht]
  \centering
  \hspace*{-1em} 
  \includegraphics[width=1.05 \columnwidth]{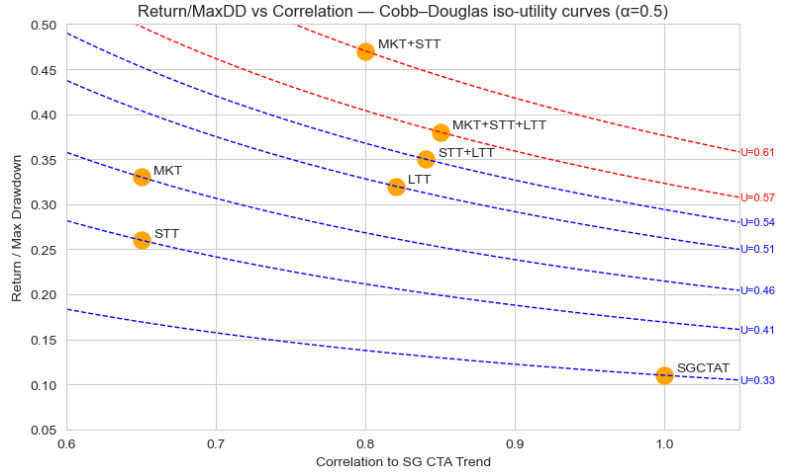}
  \caption{Iso-utility lines in the
           \((\rho,\;\mathrm{Return}/\mathrm{MaxDD})\)-plane for a
           Cobb–Douglas exponent \(\alpha = 0.50\), i.e.\ an equal
           valuation of correlation and performance.}
\end{figure}

\begin{figure}[ht]
  \centering
  \hspace*{-1em} 
  \includegraphics[width=1.05 \columnwidth]{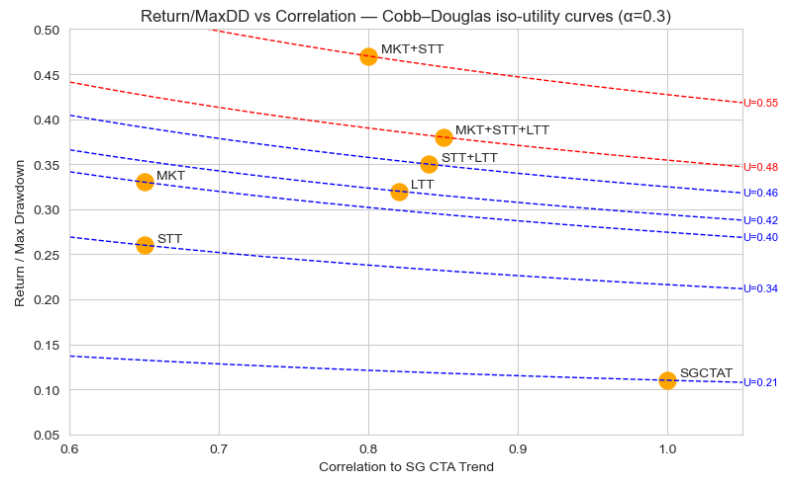}
  \caption{Iso-utility lines for \(\alpha = 0.30\), a setting that
           favours the return/MaxDD ratio over correlation.}
\end{figure}

\begin{figure}[ht]
  \centering
  \hspace*{-1em} 
  \includegraphics[width=1.05 \columnwidth]{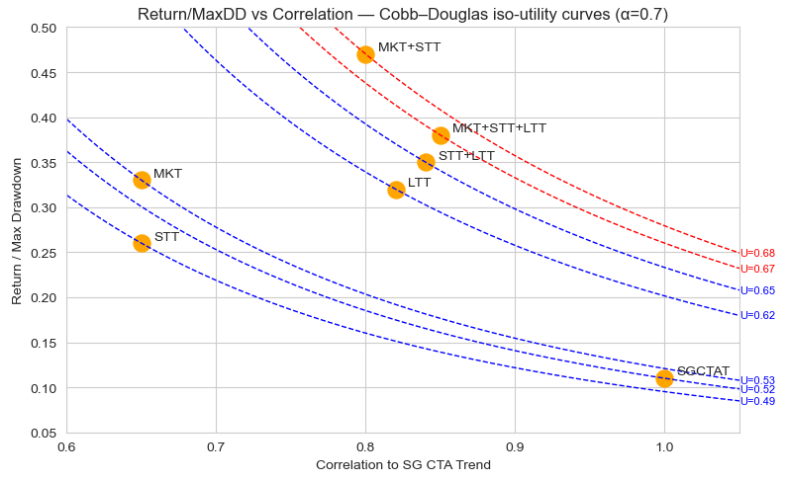}
  \caption{Iso-utility lines for \(\alpha = 0.70\), where correlation
           is given more weight than performance.}
\end{figure}

\bigskip
\noindent\textit{Comment.} Across all three previously mentioned specifications the
\textbf{MKT+STT} strategy lies above every other iso-utility curve, and
therefore dominates the alternatives.  Only when correlation is
massively overweighted—requiring \(\alpha \geq 0.78\)—would the
Market + Multi-horizon Trend blend \textbf{MKT+STT+LTT} become preferable.

\section{Conclusion}\label{sec:conclusion}

Framing Managed-Futures performance as a Bayesian state-space system gives a clear, daily-granular view of how \textit{short-term trend} (STT), \textit{long-term trend} (LTT) and \textit{market beta} (MKT) combine to drive returns and allows us answering the four questions of the introduction.  As a result, three insights emerge:

\begin{enumerate}[leftmargin=*,itemsep=3pt]
\item \textbf{Diversification value of STT.}  
      Although STT is the component with the lowest standalone Sharpe, it is the least correlated with SGCTAT, and its convex payout profile dampens drawdowns after abrupt price shocks—behaving like a low-cost “insurance layer’’ against regime shifts.
\item \textbf{Power of horizon blending.}  
      Mixing horizons is materially more efficient than choosing one.  
      A naïve 50/50 STT–LTT allocation raises the Sharpe/Max-DD efficiency ratio to \(\,2.37\) (vs.\ \(1.34\) for STT and \(2.09\) for LTT), translating into smoother equity curves and smaller peak-to-trough losses.
\item \textbf{MKT + STT dominates.}  
      Pairing raw market returns with the short-term trend factor produces a point \((\rho=0.80,\; \text{Return}/\text{MaxDD}=0.47)\) that lies on the empirical Pareto frontier as illustrated in subsection \ref{subsec:illustration}.

\end{enumerate}

\paragraph{Practical implications and future work.}
Short-horizon trend signals can feel frustrating when they whipsaw, yet our results show that their asymmetry is exactly what counter-balances the more linear risk of long-horizon rules.  In practice, the Bayesian filter gives portfolio managers a live read-out of how each horizon is contributing to return and drawdown.  Future next steps are to let the model choose its own look-back windows, compare deep-learning trend factors, and extend the study to a wider set of futures.

\clearpage
\bibliographystyle{abbrvnat}
\bibliography{main}

\clearpage
\onecolumn
\appendix
\section{Proof of the Optimum Weight Formula}\label{app:proof_optimum}
\addcontentsline{toc}{section}{Appendix A. Proof of the Optimum Weight Formula}

We wish to maximise the absolute Sharpe ratio
\begin{equation}
\mathcal S(\omega)
= \frac{\mathbb{E}[P]}%
       {\sqrt{\operatorname{Var}(P)}}
= \frac{\omega_{\mathrm{ST}}\mu_{\mathrm{ST}}
       +\omega_{\mathrm{LT}}\mu_{\mathrm{LT}}}
      {\sqrt{\omega_{\mathrm{ST}}^2
            +\omega_{\mathrm{LT}}^2
            +2\,\rho\,\omega_{\mathrm{ST}}\omega_{\mathrm{LT}}}}
\end{equation}
subject to \(\omega_{\mathrm{ST}}+\omega_{\mathrm{LT}}=1\).  Set
\(\omega_{\mathrm{LT}}=1-\omega_{\mathrm{ST}}\), and define
\begin{equation}
N(\omega_{\mathrm{ST}})
=\omega_{\mathrm{ST}}\mu_{\mathrm{ST}}
 +(1-\omega_{\mathrm{ST}})\mu_{\mathrm{LT}},
\quad
D(\omega_{\mathrm{ST}})
=\sqrt{1+2(\rho-1)\,\omega_{\mathrm{ST}}(1-\omega_{\mathrm{ST}})},
\end{equation}
so \(\mathcal S(\omega_{\mathrm{ST}})=N/D\).
\vspace{0.5cm} \\
\noindent Differentiating with respect to \(\omega_{\mathrm{ST}}\) gives
\begin{equation}
\frac{\partial \mathcal S}{\partial \omega_{\mathrm{ST}}}
=\frac{N'(\omega_{\mathrm{ST}})\,D(\omega_{\mathrm{ST}})
       -N(\omega_{\mathrm{ST}})\,D'(\omega_{\mathrm{ST}})}
      {D(\omega_{\mathrm{ST}})^2},
\end{equation}
where
\begin{equation}
N'(\omega_{\mathrm{ST}})
=\mu_{\mathrm{ST}}-\mu_{\mathrm{LT}},
\quad
D'(\omega_{\mathrm{ST}})
=\frac{(\rho-1)\,(1 - 2\omega_{\mathrm{ST}})}%
      {\sqrt{1+2(\rho-1)\,\omega_{\mathrm{ST}}(1-\omega_{\mathrm{ST}})}}.
\end{equation}
Setting the numerator to zero and solving yields 
\begin{equation}
(\mu_{\mathrm{ST}} - \mu_{\mathrm{LT}})
\bigl[1 + 2(\rho - 1)\,\omega_{\mathrm{ST}}\,(1 - \omega_{\mathrm{ST}})\bigr]
=
\bigl[\omega_{\mathrm{ST}}\,\mu_{\mathrm{ST}} + (1 - \omega_{\mathrm{ST}})\,\mu_{\mathrm{LT}}\bigr]
(\rho - 1)\,(1 - 2\,\omega_{\mathrm{ST}}).
\end{equation}

\noindent or equivalently,
\begin{align}
\text{LHS}
&= (\mu_{\mathrm{ST}} - \mu_{\mathrm{LT}})
   \bigl[1 + 2(\rho - 1)\,\omega_{\mathrm{ST}} - 2(\rho - 1)\,\omega_{\mathrm{ST}}^2\bigr] \\
&= (\mu_{\mathrm{ST}} - \mu_{\mathrm{LT}})
   + 2(\rho - 1)(\mu_{\mathrm{ST}} - \mu_{\mathrm{LT}})\,\omega_{\mathrm{ST}}
   - 2(\rho - 1)(\mu_{\mathrm{ST}} - \mu_{\mathrm{LT}})\,\omega_{\mathrm{ST}}^2,
\\[1ex]
\text{RHS}
&= \bigl[\omega_{\mathrm{ST}}\mu_{\mathrm{ST}} + (1 - \omega_{\mathrm{ST}})\mu_{\mathrm{LT}}\bigr]
   (\rho - 1)(1 - 2\omega_{\mathrm{ST}})\\
&= (\rho - 1)\,\mu_{\mathrm{LT}}
   + (\rho - 1)(\mu_{\mathrm{ST}} - 3\mu_{\mathrm{LT}})\,\omega_{\mathrm{ST}}
   - 2(\rho - 1)(\mu_{\mathrm{ST}} - \mu_{\mathrm{LT}})\,\omega_{\mathrm{ST}}^2.
\end{align}

\noindent By matching coefficients of \(1\), \(\omega_{\mathrm{ST}}\), and \(\omega_{\mathrm{ST}}^2\), we obtain:

\begin{equation}
\begin{cases}
\text{const. term:} & \mu_{\mathrm{ST}} - \mu_{\mathrm{LT}} = (\rho - 1)\,\mu_{\mathrm{LT}},\\[.5em]
\text{linear term:} & 2(\rho - 1)(\mu_{\mathrm{ST}} - \mu_{\mathrm{LT}})
                     = (\rho - 1)(\mu_{\mathrm{ST}} - 3\mu_{\mathrm{LT}}),\\[.5em]
\text{quadratic term:} & -2(\rho - 1)(\mu_{\mathrm{ST}} - \mu_{\mathrm{LT}})
                        = -2(\rho - 1)(\mu_{\mathrm{ST}} - \mu_{\mathrm{LT}}).
\end{cases}
\end{equation}

\noindent The quadratic term cancels out and we are left with:
\begin{equation}
\underbrace{(\mu_{\rm ST}-\mu_{\rm LT}) - (\rho-1)\,\mu_{\rm LT}}_{\text{constant part}}
\;=\;
\underbrace{(\rho-1)(\mu_{\rm ST}-3\mu_{\rm LT})
  - 2(\rho-1)(\mu_{\rm ST}-\mu_{\rm LT})}_{\text{coefficient of }\omega_{\rm ST}}
\;\omega_{\rm ST}.
\end{equation}

\noindent Solve for \(\omega_{\rm ST}\):
\begin{equation}
\omega_{\rm ST}
= \frac{(\mu_{\rm ST}-\mu_{\rm LT}) - (\rho-1)\,\mu_{\rm LT}}
       {(\rho-1)(\mu_{\rm ST}-3\mu_{\rm LT})
        - 2(\rho-1)(\mu_{\rm ST}-\mu_{\rm LT})}
= \frac{\mu_{\rm ST}-\rho\,\mu_{\rm LT}}
       {-(\rho-1)\,(\mu_{\rm ST}+\mu_{\rm LT})}.
\label{eq:omega_temp}
\end{equation}

\noindent Or equivalently, the unique maximizer
\begin{equation}
\omega_{\mathrm{ST}}^*
=\frac{\mu_{\mathrm{ST}}-\rho\,\mu_{\mathrm{LT}}}
      {(\mu_{\mathrm{ST}}+\mu_{\mathrm{LT}})(1-\rho )},
\end{equation}
which is exactly formula~\eqref{eq:optimum}.  \qed

\section{Supplementary Performance Tables}
\label{app:perf_tables}
To gauge the stability of each sleeve ( LTT   \& MKT   \& STT+LTT \& STT   \& MKT+STT+LTT \& MKT+STT \& SGCTAT) across distinct market back-drops, we
report performance over four horizons: 

\begin{enumerate}[label=(\Alph*)]
  \item \textbf{Live window (Jun 2020–Jun 2025).}  
        The most recent five-year, strictly out-of-sample slice covers the
        post-COVID rebound, the 2022 inflation shock and the 2024–25 tariff
        turbulence, providing the clearest picture of how the sleeves behave
        under current market conditions.

  \item \textbf{Full back-test (Jan 2010–Jun 2025).}  
        The entire 15½-year data set aggregates every daily observation we
        have, furnishing the broadest statistical base for long-run
        risk-return estimates.

  \item \textbf{Post-GFC bull run (Jan 2010–Dec 2015).}  
        A trend-friendly environment characterised by QE-fuelled rallies and
        steady macro momentum—ideal for gauging performance when markets move
        persistently in one direction.

  \item \textbf{Pre-pandemic transition (Jan 2016–Dec 2020).}  
        A flatter, more volatile phase dominated by rate-hike cycles, trade-war
        headlines and the 2020 COVID crash, useful for testing robustness when
        price action turns choppy and mean-reverting.
\end{enumerate}

Each panel—Tables \ref{tab:perf_last5}–\ref{tab:perf_2016_2020}—lists
cumulative and annual returns, volatility, Sharpe ratio, maximum drawdown
(Max-DD) and the return-to-drawdown efficiency (Ret/Max-DD) for every factor
portfolio alongside the SG CTA Trend benchmark.

\begin{table}[H]
  \centering
  \caption{Panel A: Performance metrics over the live five-year window (Jun 2020–Jun 2025).}
  \label{tab:perf_last5}
  \setlength{\tabcolsep}{4pt}
  \renewcommand{\arraystretch}{1.1}
  \begin{tabular}{@{}lrrrrrrr@{}}
    \toprule
    Metric & LTT & MKT & STT+LTT & STT & MKT+STT+LTT & MKT+STT & SGCTAT \\
    \midrule
    Cumulative Return (\%) & 46.9 & 60.6 & 39.0 & 30.3 & 41.0 & 48.6 & 33.3 \\
    Annual Return (\%)     &  8.0 &  9.9 &  6.8 &  5.4 &  7.1 &  8.2 &  5.9 \\
    Volatility (\%)        & 10.8 & 11.6 & 10.2 &  9.9 & 11.0 & 10.9 & 10.9 \\
    Sharpe Ratio           & 0.46 & 0.59 & 0.37 & 0.25 & 0.38 & 0.48 & 0.27 \\
    Max Drawdown (\%)      & 18.8 & 20.3 & 16.7 & 15.2 & 17.7 & 14.9 & 21.7 \\
    Return/Max-DD          & 0.43 & 0.67 & 0.41 & 0.36 & 0.40 & 0.55 & 0.27 \\
    \bottomrule
  \end{tabular}
\end{table}

\begin{table}[H]
  \centering
  \caption{Panel B: Full-sample performance (Jan 2010–Jun 2025).}
  \label{tab:perf_full}
  \setlength{\tabcolsep}{4pt}
  \renewcommand{\arraystretch}{1.1}
  \begin{tabular}{@{}lrrrrrrr@{}}
    \toprule
    Metric & LTT & MKT & STT+LTT & STT & MKT+STT+LTT & MKT+STT & SGCTAT \\
    \midrule
    Cumulative Return (\%) & 259.1 & 246.7 & 267.0 & 105.6 & 313.1 & 308.9 & 55.6 \\
    Annual Return (\%)     &   8.6 &   8.4 &   8.8 &   4.8 &   9.6 &   9.5 &  2.9 \\
    Volatility (\%)        &  10.3 &  11.4 &   9.6 &   9.3 &  10.3 &  10.1 & 11.0 \\
    Sharpe Ratio           &  0.70 &  0.61 &  0.76 &  0.37 &  0.80 &  0.80 & 0.14 \\
    Max Drawdown (\%)      &  18.8 &  20.3 &  16.7 &  15.2 &  17.7 &  14.9 & 23.0 \\
    Return/Max-DD          &  0.46 &  0.41 &  0.53 &  0.31 &  0.54 &  0.64 & 0.13 \\
    \bottomrule
  \end{tabular}
\end{table}

\begin{table}[H]
  \centering
  \caption{Panel C: Post-GFC regime (Jan 2010–Dec 2015).}
  \label{tab:perf_2010_2015}
  \setlength{\tabcolsep}{4pt}
  \renewcommand{\arraystretch}{1.1}
  \begin{tabular}{@{}lrrrrrrr@{}}
    \toprule
    Metric & LTT & MKT & STT+LTT & STT & MKT+STT+LTT & MKT+STT & SGCTAT \\
    \midrule
    Cumulative Return (\%) & 110.3 &  97.1 & 115.4 & 34.8 & 126.2 & 124.1 & 23.3 \\
    Annual Return (\%)     &  13.2 &  12.0 &  13.6 &  5.1 &  14.6 &  14.4 &  3.6 \\
    Volatility (\%)        &  10.5 &  11.2 &   9.8 &  9.4 &  10.2 &  10.1 & 10.9 \\
    Sharpe Ratio           &  1.25 &  1.06 &  1.38 & 0.53 &  1.41 &  1.41 & 0.31 \\
    Max Drawdown (\%)      &  11.1 &  14.0 &   9.3 & 12.9 &   9.8 &  11.8 & 17.7 \\
    Return/Max-DD          &  1.19 &  0.86 &  1.47 & 0.40 &  1.48 &  1.22 & 0.20 \\
    \bottomrule
  \end{tabular}
\end{table}

\begin{table}[H]
  \centering
  \caption{Panel D: Pre-pandemic transition (Jan 2016–Dec 2020).}
  \label{tab:perf_2016_2020}
  \setlength{\tabcolsep}{4pt}
  \renewcommand{\arraystretch}{1.1}
  \begin{tabular}{@{}lrrrrrrr@{}}
    \toprule
    Metric & LTT & MKT & STT+LTT & STT & MKT+STT+LTT & MKT+STT & SGCTAT \\
    \midrule
    Cumulative Return (\%) & 26.7 &  9.3 & 34.7 & 25.1 & 40.3 & 27.8 &  2.3 \\
    Annual Return (\%)     &  4.8 &  1.8 &  6.1 &  4.6 &  7.0 &  5.0 &  0.5 \\
    Volatility (\%)        &  9.2 & 11.3 &  8.5 &  8.1 &  9.3 &  9.1 & 11.0 \\
    Sharpe Ratio           &  0.40 & 0.05 & 0.58 & 0.42 & 0.62 & 0.42 & -0.06 \\
    Max Drawdown (\%)      & 12.2 & 20.3 &  8.3 &  8.4 &  9.9 & 10.8 & 22.4 \\
    Return/Max-DD          &  0.40 & 0.09 & 0.74 & 0.54 & 0.70 & 0.47 & 0.02 \\
    \bottomrule
  \end{tabular}
\end{table}

\section{Cost–Calculation Details}\label{app:cost-details}

We provide here the cost framework used in
Section~\ref{sec:cost-summary}.  We move from \emph{contract-level inputs} to
the \emph{all-in portfolio drag} quoted in basis points (bp) of assets under
management (AUM) per annum.

\subsection{Implementation Cost Components}

\subsubsection{Transaction Cost per Round-Turn}\label{app:tx-cost}

For each futures contract \(i\) the execution cost of a \emph{round-turn}
(open+close) is expressed in bp of notional:
\[
  C^{\text{Tx}}_{i}
  \;=\;
  \frac{
    \tfrac{\text{Bid–Ask Spread}}{2}
    \;+\;\text{Brokerage}
    \;+\;\text{Exchange \& Clearing Fees}
    \;+\;\text{Slippage Buffer}
  }{\text{Contract Notional}}\times10^{4}.
\]
\begin{itemize}[itemsep=1pt, topsep=2pt]
  \item \textbf{Bid–Ask spread.} Daily median quoted half-spread (ticks) from
        Bloomberg TOMS over 2000-2025.
  \item \textbf{Brokerage.} Volume-weighted average commission reported by
        three tier-one FCMs.
  \item \textbf{Exchange/clearing fees.} Schedule as of 2025\,Q1.
  \item \textbf{Slippage buffer.} \(0.5\,\sigma_{\text{1-min}}\) where
        \(\sigma_{\text{1-min}}\) is the one-minute return volatility,
        capturing impact in stressed tape.
\end{itemize}

\paragraph{Why a 2 bp round-turn cost?}
We illustrate the calculation for the E-mini S\&P 500 (ES) contract:
\[
\begin{aligned}
\text{Half-spread} &\,=\, 0.25\,\text{pt} = \$12.50, \\
\text{Brokerage} &\,=\, \$1.18, \\
\text{Fees} &\,=\, \$0.85, \\
\text{Slippage buffer} &\,=\, \$3.00, \\
\text{Notional} &\,=\, \$50 \times \underline{4\,800} = \$240\,000. \\
\therefore C^{\text{Tx}}_{\text{ES}} &= \frac{12.50+1.18+0.85+3.00}{240\,000}\times10^{4} \approx \mathbf{0.75\;bp}.
\end{aligned}
\]

Across all 24 contracts the \emph{volume-weighted} average is
\(\bar C^{\text{Tx}}=2\,\text{bp}\) per round-turn, matching
Table~\ref{tab:futures_universe}. We therefore adopt
\(\bar C^{\text{Tx}} = 2\,\text{bp}\) as a single, all-in execution hair-cut to keep things simple, though quite  realistic—if anything conservative.

\subsubsection{Roll Cost per Contract}\label{app:roll-cost}

Replication drag arises when the front-month future \((F_{0})\) is replaced by
the next-month \((F_{1})\).  We proxy the carry over 2000-2025 by
\[
  C^{\text{Roll}}_{i}
  \;=\;
  \bigl[\ln(F_{1}/F_{0})\bigr]_{\text{avg}}\times10^{4}\;\text{bp}.
\]
Values cluster around \(12\,\text{bp}\) for rates and metals, \(15\,\text{bp}\)
for equity indices and energies, and \(2\,\text{bp}\) for major currency pairs
(where the front-next basis is negligible).

\subsubsection{Management Fee}\label{app:mgmt-fee}

We charge a flat
\[
  C^{\text{Mgmt}} = 50\,\text{bp \; p.a.}
\]
on AUM—roughly one-third of the classic “2\%+20\%” fee stack still common
among discretionary CTAs.

\subsection{Putting the Numbers in Context}
The strategies presented in the paper employ an average gross leverage of \textbf{4×} AUM and—based on empirical fills from our execution stack—generate between \textbf{20 and 35} round-turns per year, depending on the horizon.

\begin{itemize}
  \item \textbf{Transaction cost drag}: \(20\text{–}35\times2\,\mathrm{bp}=40\text{–}70\,\mathrm{bp}\) per year.
  \item \textbf{Roll drag}: \(4\times12\,\mathrm{bp}=48\,\mathrm{bp}\) per year (on average).
  \item \textbf{Management fee}: fixed 0.50\% p.a. (50 bp).
\end{itemize}

Combining these elements yields an \emph{all-in} cost of
\[
40 + 48 + 50 = 138\,\mathrm{bp}
\quad\text{to}\quad
70 + 48 + 50 = 168\,\mathrm{bp},
\]
i.e.\ \textbf{1.38\% to 1.68\%} per annum, reflecting the costs associated with transactions, roll drag, and moderate management fees in each tested strategy.

\end{document}